\documentclass[12pt]{article}
\usepackage{latexsym,amsmath,amssymb,graphicx}

\newtheorem{thm}{Theorem}[section]
\newtheorem{lemma}[thm]{Lemma}

\newtheorem{fig}{Figure}
\newtheorem{tab}{Table}
\newenvironment{proof}{\begin{trivlist}\item[]{\em Proof.\/\ }}%
                      {\hfill$\Box$ \\ \end{trivlist}}

\title{A structure from motion inequality}
\author{Oliver Knill and Jose Ramirez-Herran 
\footnote{Harvard University, this research was supported by the Harvard Extension School}}
\date{August 17, 2007}
\begin{document}
\maketitle

\abstract{
We state an elementary inequality for the structure from motion problem for 
$m$ cameras and $n$ points. This structure from motion inequality relates space dimension, 
camera parameter dimension, the number of cameras and number points and 
global symmetry properties and provides a 
rigorous criterion for which reconstruction is not possible with probability 1. 
Mathematically the inequality is based on Frobenius theorem which is a geometric 
incarnation of the fundamental theorem of linear algebra. The paper also provides a
general mathematical formalism for the structure from motion problem. It 
includes the situation the points can move while the camera 
takes the pictures. 
}

\section{Introduction}

The {\bf structure from motion problem} is the task to reconstruct
$m$ cameras and $n$ points from the $nm$ pictures which the cameras have taken. It is
a central problem in {\bf computer vision} \cite{Ullman,forsyth, hartley,trucco}. \\

We define a {\bf camera} as a piecewise smooth map $Q$ from a $d$-dimensional space 
$N$ to $N$ satisfying $Q^2=Q$ such that $Q(N)$ is a lower dimensional surface,
the {\bf retinal surface}.
We assume that for a given camera type, the set $M$ of all possible cameras 
is a manifold of finite dimension $f$ and that the manifolds $Q(N)$ are all diffeomorphic to 
a fixed manifold $S$. Given $n$ points $P_i$ in $N$ and $m$ points $Q_j$ in $M$,
the problem is to reconstruct the cameras $Q_j$ and the locations $P_i$ of the points 
from the {\bf image data} $Q_j(P_i)$. The map $F$ from $N^n \times M^n \to S^{m n}$ is called
the {\bf structure from motion map}. It is in general nonlinear. We assume that $F$ is real analytic
on an open subset of $N^n \times M^n$. \\

This reconstruction should be {\bf locally unique} after 
factoring out {\bf global symmetries} like for example a common translation of both cameras and points. 
Global symmetries are in general a Lie group $G$ of dimension $g$. The {\bf global symmetry group $G$}
acts on $M \times N$ in such a way that $\gamma(Q) ( \gamma(P) ) = Q(P)$ for every $Q \in M$,
$P \in N$ and $\gamma \in G$.  \\

The name "structure from motion" originates from a different point of view:
fix a single camera and "move" $n$ points by an Euclidean rigid motion. 
The camera then takes $m$ pictures of this moving body. The aim is to reconstruct 
the location of the points and the deformation path of the points.
While this second point of view motivates situations, where 
the points undergo a non-rigid motion or a rigid motion satisfying some constraints 
like angular momentum and energy conservation, we will stick to the first formulation, which 
allows us to includes examples, where the moving camera changes internal camera parameters 
like the focal length while shooting the pictures. \\

A basic question is to find the minimal number of cameras for a given point set or the
minimal number of points for a given number of cameras so that we have a locally 
unique reconstruction. This motivates to look for explicit inversion formulas 
for the structure from motion map $F$ as well as the exploration of ambiguities: 
camera-point configurations which have the same image data. \\

Our formalism is quite general. The configuration manifold $N$ of a single point can for example 
be a finite dimensional manifold of curves. An example would be a situation, where 
$N$ is a $(k+1) \cdot d$-dimensional 
manifold of $k$-jets describing moving bodies with Taylor expansion $P_i(t)=\sum_{l=0}^{k-1} P_{il} t^l$. 
The structure from motion problem in that case is to invert $F$ that is to 
reconstruct the moving points $P_i(t)$ and the cameras $Q_j$ from the camera pictures 
$Q_j(P_i(t_j))$. A concrete example would be a camera mounted on a car taking pictures 
during the drive. The task is to reconstruct not only the surrounding and the path of 
the camera but the motion of the other cars. \\

The mathematics involved in the structure from motion problem depends on the 
camera model and the point model. The former is represented by the {\bf camera parameter manifold} 
$M$ and the later is modeled by the {\bf point manifold} $N$. Each camera
$Q$ is a map from $N$ to a lower-dimensional surface $S \subset N$, the {\bf retinal surface}. 
We want to invert the map $F: N^n \times M^m \to S^{mn}$ modulo a global 
symmetry group $G$ acting on $N$ and $M$. The structure from motion inequality allows us to see
for which $m$ and $n$, the image of the map $F$ has full dimension and so 
that $F^{-1}(\sigma)$ is a discrete set. For example, for orthographic cameras in space, 
$m=3$ cameras and $n=3$ points lead to a locally unique reconstruction 
\cite{KnillRamirezUllman}. For $m=3$ cameras and $n=4$ points, where Ullman's theorem
leads to a unique reconstruction modulo reflection, we have an over-determined system. In general, 
an over-determined system assures the injectivity of $F$ but $F(N)$ is a lower dimensional surface in 
$S^{n m}$. 

\section{Examples of cameras}

% perspective cameras 
A {\bf perspective camera} in three dimensional space is defined by a point $C$, the {\bf center of projection}
and a plane $S$, the {\bf retinal plane}. Perspective cameras are also called {\bf pinhole cameras}. 
A point $P$ in space is mapped to a point $p=Q(P)$ on $S$ by intersecting the line $CP$ with $S$. 
Perspective cameras can also be defined in the plane where they are 
defined by a point and a line. Limiting cases of perspective cameras are {\bf affine cameras} for which 
parallelism is preserved. Special cases are {\bf weak perspective cameras} and more specially 
{\bf orthographic affine cameras}, where the image is the orthogonal projection of space onto the retinal plane.
A weak perspective camera is an orthogonal projection onto a plane combined with an additional scaling in that plane.
One can think of orthographic cameras as pinhole cameras with the center of projection $C$ is at infinity. 
For affine cameras, the observers position is not determined. So, even if we reconstruct camera and point 
positions, we will not know, where the pictures were taken. An other perspective camera is the 
{\bf push-broom camera} \cite{hartley} 
which is defined by a line $L$ in space and a plane $S$ parallel to the line.
A point $P$ in space is projected onto the line. The image point $Q_(P)$ is the intersection of the 
projection line with the plane. \\

% omni-directional cameras 
A {\bf spherical camera} in space is defined by a point $C$ and a sphere $S$ centered
at $C$. The map $Q$ maps $P$ to a point $p=Q(P)$ on $S$ by intersecting the line $CP$
with $S$. We label a point $p$ with two spherical Euler angles $(\theta,\phi)$. We also use the more 
common name {\bf omni-directional cameras} or {\bf central panoramic cameras}.
In two dimensions, one can consider {\bf circular camera} defined by a point $C$ and a circle $S$ around the
point. A point $P$ in the plane is mapped onto a point $p$ on $S$ by intersecting the line $CP$
with $S$. Spherical and circular cameras only have the point $C$ and the 
orientation as internal parameters. The radius of the sphere is irrelevant. 
A {\bf cylindrical cameras} in space is defined by a point $C$ and a cylinder $C$ with axes $L$.
A point $P$ is mapped to the point $p$ on $C$ which is the intersection of the line
$C P$ with $C$. A point $p$ in the film surface $S$ can be described 
with cylinder coordinates $(\theta,z)$. 
Because cylindrical cameras capture the entire world except for points on the 
symmetry axes of the cylinder, one could include them in the class of {\bf omni-directional cameras}. 
Omni-directional camera pictures are also called {\bf panoramas}, even if only part of the
360 field of view and part of the height are known.
Cylindrical and spherical cameras are closely related. The Euler angle 
$\phi$ between the line $CP$ and the horizontal plane and the radius $r$ of the cylinder, 
gives the height $z=r \sin(\phi)$, so that a simple change of the coordinate system matches
one situation with the other. We can also remap the picture of a perspective camera to be part 
of an omni-directional camera picture, like if a small part of the sphere is replaced by a region in its 
tangent plane. Because spherical cameras do not have a focal parameter $f$ as perspective cameras, 
they are easy to work with. For more information, see \cite{Benosman}. 
% orientation of the camera 
We say, a spherical camera is {\bf oriented}, if its direction is known. Oriented spherical cameras
have only the center of the camera as their internal parameter. The parameter space is therefore 
$d$-dimensional. For non-oriented spherical cameras, there are additionally $d (d-1)/2 = {\rm dim}(SO_d)$ 
parameters to describe the orientation of the camera. 

\begin{center}
\parbox{15.5cm}{
\parbox{4.8cm}{\scalebox{0.45}{\includegraphics{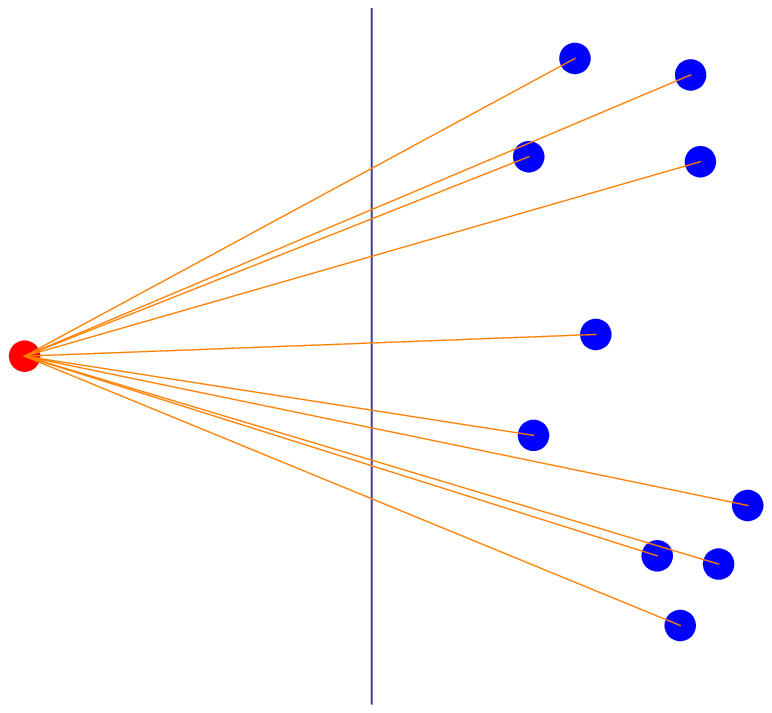}}}
\parbox{4.8cm}{\scalebox{0.45}{\includegraphics{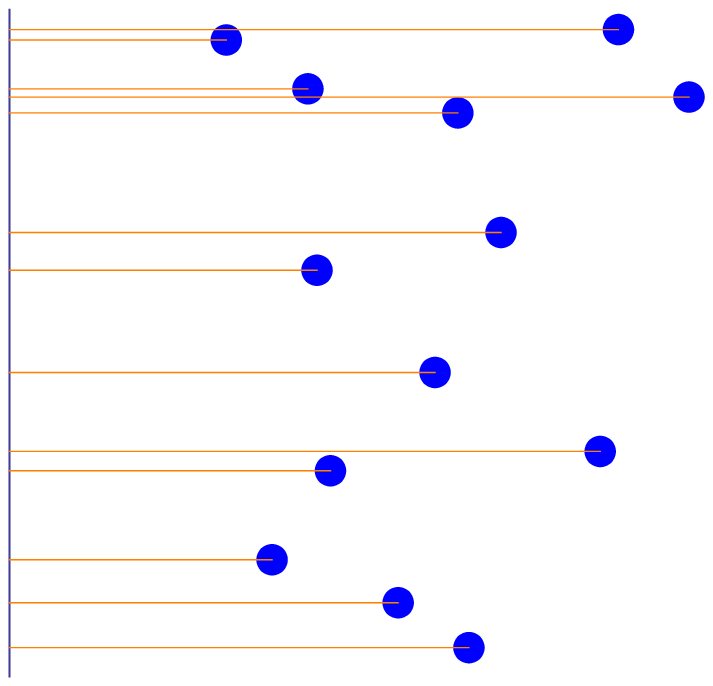}}}
\parbox{4.8cm}{\scalebox{0.45}{\includegraphics{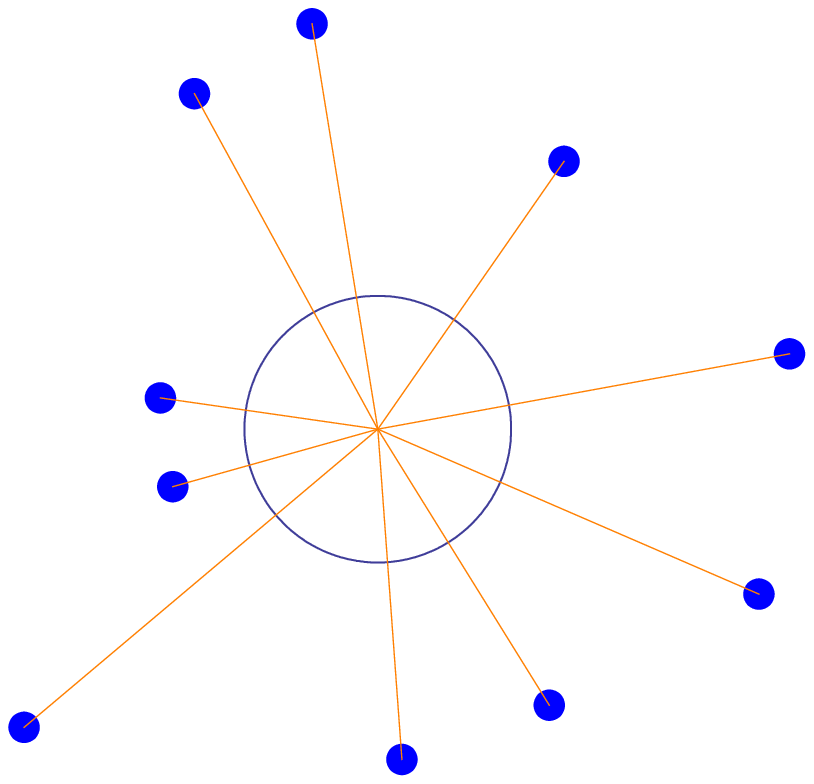}}}
}
\end{center}
\begin{center}
\parbox{15.5cm}{
\parbox{4.8cm}{ Perspective onto a line}
\parbox{4.8cm}{ Orthographic onto a line}
\parbox{4.8cm}{ Circular onto a circle}
}
\end{center}

\begin{center}
\parbox{15.5cm}{
\parbox{4.8cm}{\scalebox{0.45}{\includegraphics{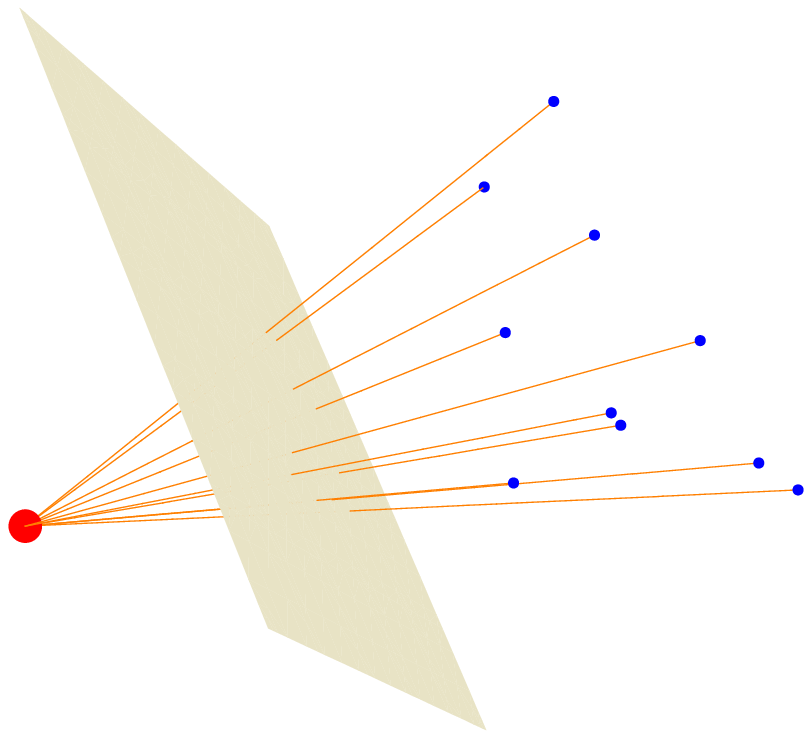}}}
\parbox{4.8cm}{\scalebox{0.45}{\includegraphics{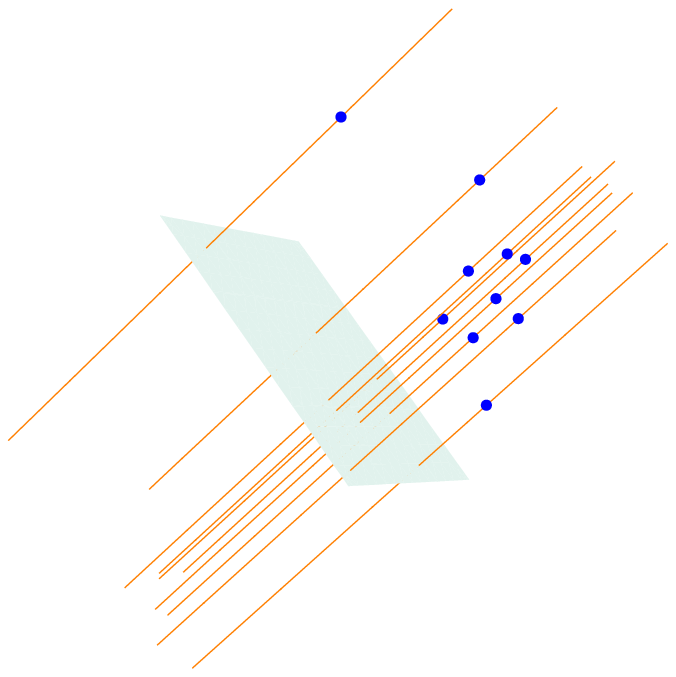}}}
\parbox{4.8cm}{\scalebox{0.45}{\includegraphics{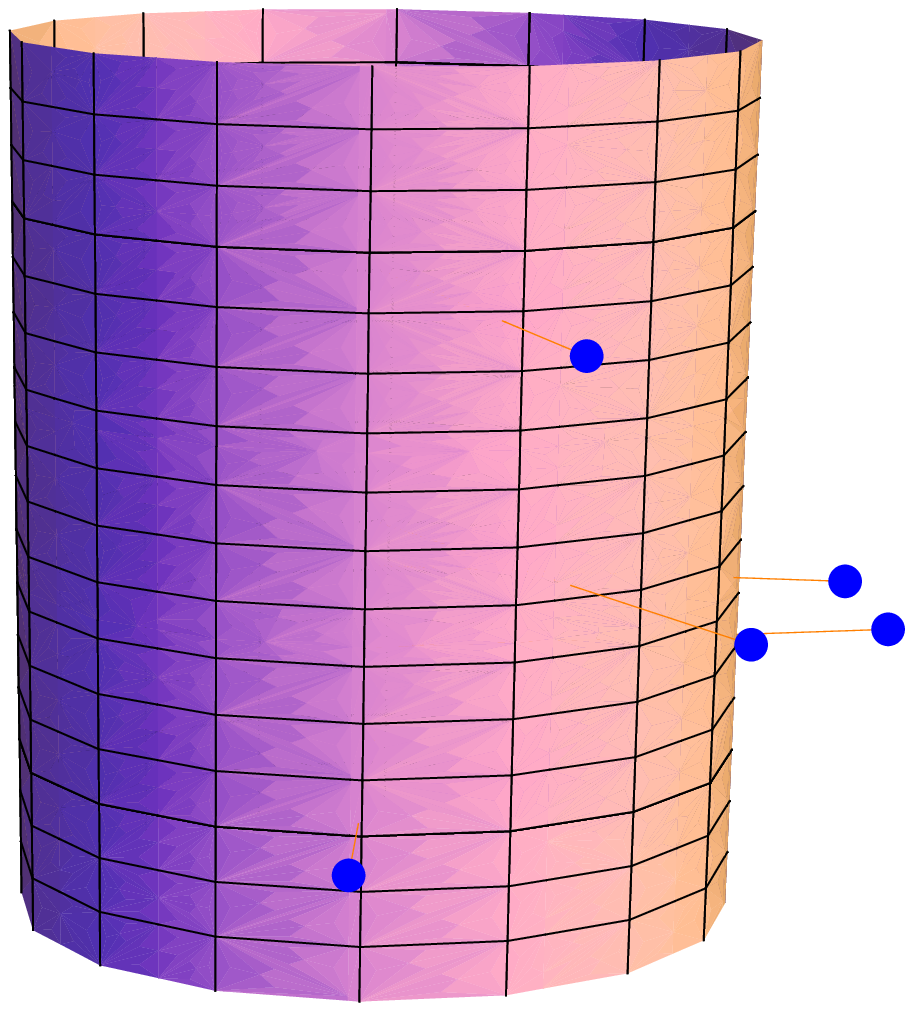}}}
}
\end{center}
\begin{center}
\parbox{15.5cm}{
\parbox{4.8cm}{ Perspective onto a plane}
\parbox{4.8cm}{ Orthographic onto a plane}
\parbox{4.8cm}{ Cylindrical onto a cylinder}
}
\end{center}

\begin{center}
\parbox{15.5cm}{
\parbox{4.8cm}{\scalebox{0.45}{\includegraphics{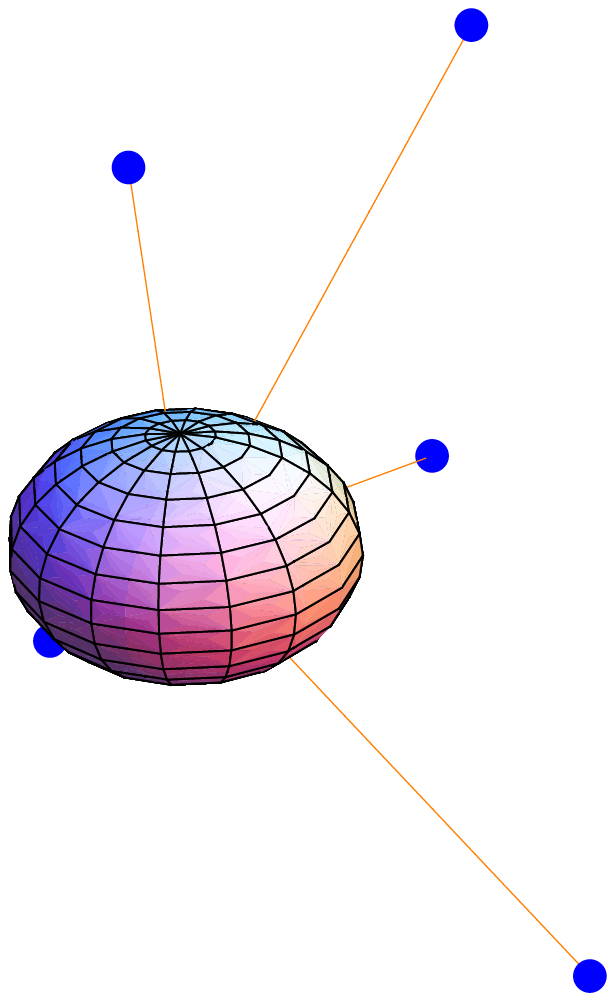}}}
\parbox{4.8cm}{\scalebox{0.45}{\includegraphics{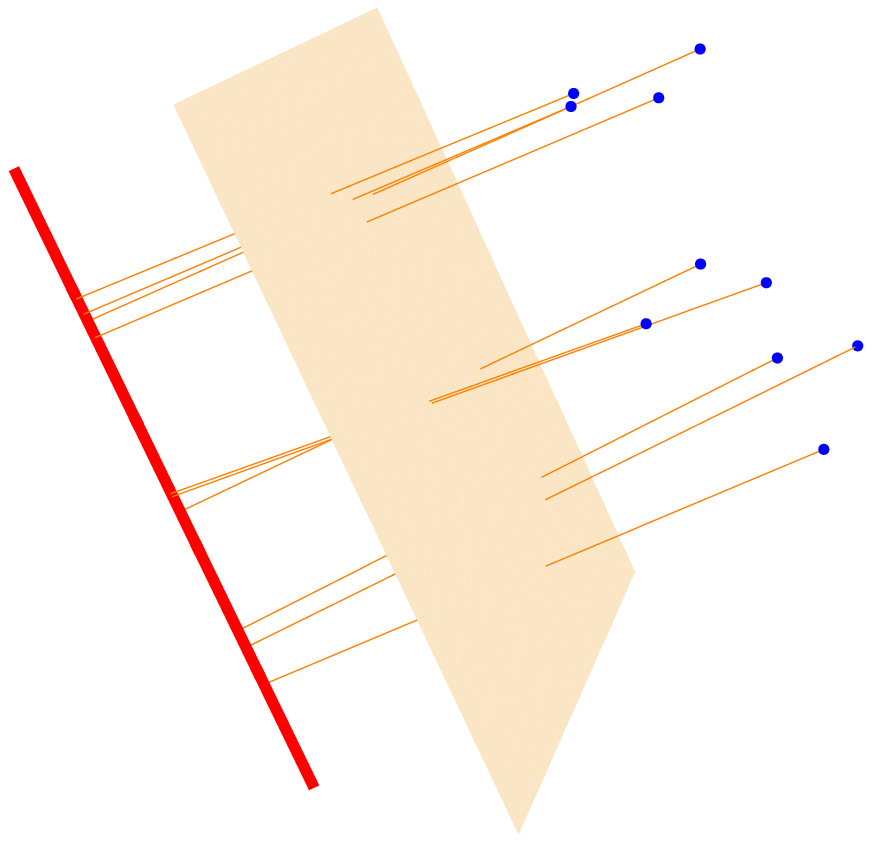}}}
\parbox{4.8cm}{\scalebox{0.45}{\includegraphics{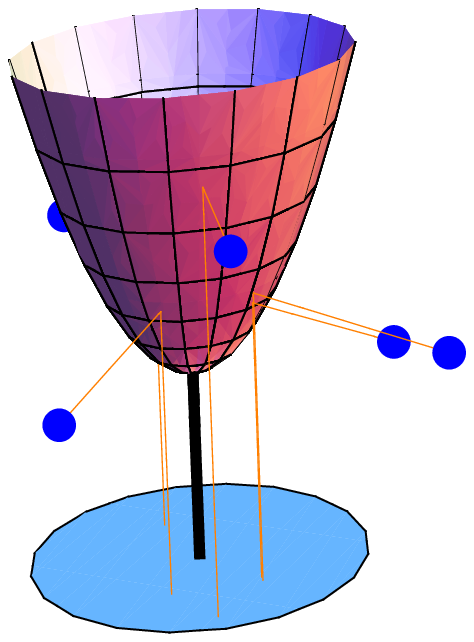}}}
}
\end{center}
\begin{center}
\parbox{15.5cm}{ 
\parbox{4.8cm}{ Spherical onto a sphere}
\parbox{4.8cm}{ Push-broom onto a plane}
\parbox{4.8cm}{ Catadioptric onto a plane}
}
\end{center}

\begin{fig}
Various examples of cameras, maps $Q:N \to N$ satisfying $Q^2=Q$ which have as an image
a hypersurface $Q(N) \subset N$ diffeomorphic to a manifold $S$.
\end{fig}

\section{The camera parameter manifold}

The number of parameters $f$ which determine a camera is the dimension of the 
{\bf camera parameter manifold} $M$.  \\

{\bf Examples:} \\

\noindent
1. For an affine camera in 
$d$ dimensional space, where only the orientation with $d (d-1)/2={\rm dim}(SO_3)$ 
parameters and a translation in the $(d-1)$-dimensional retinal plane with $d-1$ parameters matters, 
the total number of parameters is $f=d (d-1)/2 + (d-1)$. \\
2. For an oriented omni-directional camera, we only 
need to know the position so that $f=d$. \\
3. For a non-oriented omni-directional camera, we need to know additionally the 
orientation which lead to $f=d+d (d-1)/2$ parameters. \\
4. For perspective pinhole cameras, we have have to specify the center of projection,
the plane orientation and the distance of the plane to the point: 
$d + d (d-1)/2 + 1$ which are 7 parameters in $d=3$ dimensions. \\

Let's look first at cameras in the two-dimensional plane. We use the notation $R^2_0=R^2 \setminus \{0,0\}$
for the punctured plane, $S^1$ for the one-dimensional circle. 

\begin{center}
\begin{tabular}{|c|c|c|c|c|c|} \hline
\parbox{1.2cm}{   } & 
\parbox{2.2cm}{affine} & 
\parbox{2.2cm}{omni} &   
\parbox{2.2cm}{n.o.-omni} &  
\parbox{2.2cm}{perspective} & 
\parbox{2.2cm}{f-perspective} \\ \hline
\parbox{1.2cm}{$M=$} & 
\parbox{2.2cm}{$S^1 \times R$} & 
\parbox{2.2cm}{$R^2$} & 
\parbox{2.2cm}{$R^2 \times S^1$} & 
\parbox{2.2cm}{$R^2 \times S^1$} & 
\parbox{2.2cm}{$R^2 \times R^2_0$}  \\ \hline
\parbox{1.2cm}{$f=$} & 
\parbox{2.2cm}{$1+1=2$} & 
\parbox{2.2cm}{$2$} & 
\parbox{2.2cm}{$2+1=3$} & 
\parbox{2.2cm}{$2+1=3$} & 
\parbox{2.2cm}{$2+2=4$}  \\ \hline
\parbox{1.2cm}{   } & 
\parbox{2.2cm}{line slope and origin} &
\parbox{2.2cm}{camera position} &
\parbox{2.2cm}{position and orientation} &
\parbox{2.2cm}{projection center and line slope} &
\parbox{2.2cm}{projection center and line} \\ \hline
\end{tabular}
\end{center}

\begin{tab}
The dimension $f$ of the camera parameter space $M$ in the two-dimensional case $d=2$ for various cameras.
\end{tab}

In the following table, $SO_3$ is the group of all rotations in space. It is a three dimensional Lie group.

\begin{center}
\begin{tabular}{|c|c|c|c|c|c|} \hline
\parbox{1.2cm}{   } & 
\parbox{2.2cm}{affine} & 
\parbox{2.2cm}{omni} &   
\parbox{2.2cm}{n.o.-omni} &  
\parbox{2.2cm}{perspective} & 
\parbox{2.2cm}{f-perspective} \\ \hline
\parbox{1.2cm}{$M=$} & 
\parbox{2.2cm}{$SO_3 \times R^2$} & 
\parbox{2.2cm}{$R^3$} & 
\parbox{2.2cm}{$R^3 + SO_3$} & 
\parbox{2.2cm}{$R^3 \times SO_3$} & 
\parbox{2.2cm}{$R^3 \times SO_3 \times R^+$}  \\  \hline
\parbox{1.2cm}{$f=$} & 
\parbox{2.2cm}{$3+2=5$} & 
\parbox{2.2cm}{$3$} & 
\parbox{2.2cm}{$3+3=6$} & 
\parbox{2.2cm}{$3+3=6$} & 
\parbox{2.2cm}{$3+3+1=7$}  \\  \hline
\parbox{1.2cm}{   } & 
\parbox{2.2cm}{Plane orientation and origin} &
\parbox{2.2cm}{Position of camera} &
\parbox{2.2cm}{Position and orientation} &
\parbox{2.2cm}{Projection center and plane orientation} &
\parbox{2.2cm}{Projection center and plane orientation and distance} \\ \hline
\end{tabular}
\end{center}

\begin{tab}
The dimension $f$ of the camera parameter space $M$ in the three dimensional case $d=3$ for various cameras.
\end{tab}

\section{The point-camera symmetry group}

Depending on the camera, there is global symmetry group $G$ for the structure from motion problem. It acts
on $N$ and $M$. If an element of this group is applied to the point camera positions simultaneously, 
the photographer produces the same photographs. In other words, if $(P',Q')$ is obtained by applying an 
element of $G$ on $(P,Q)$, then $Q_j(P_i) = Q'_j(P'_i)$ for all $1 \leq i \leq n, 1 \leq j \leq m$. 
We call $G$ the {\bf point-camera symmetry group} and denote its dimension by $g$. \\

{\bf Examples.} \\

\noindent
1) For affine cameras in $N=R^d$, the group is the Euclidean group $R^d \times SO_d$ which has
dimension $d+ d(d-1)/2$. \\
2) For oriented spherical cameras, the group consists of dilations. The group of transformations generated by
 translations and scalings. Its dimension is $d+1$. \\
3) For non-oriented spherical cameras, the group consists of all similarities, which are generated by
Euclidean transformations and scalings. The dimension is $d+ d(d-1)/2+1$.  \\
4) For perspective cameras, we just have the group of Euclidean transformations as the global symmetry group. \\

Let's look at some cameras in two dimensions first: 

\begin{center}
\begin{tabular}{|c|c|c|c|c|c|} \hline
\parbox{1.2cm}{ } & 
\parbox{2.2cm}{affine} & 
\parbox{2.2cm}{omni} &   
\parbox{2.2cm}{n.o.-omni} &  
\parbox{2.2cm}{perspective} & 
\parbox{2.2cm}{f-perspective} \\ \hline
\parbox{1.2cm}{$G=$} & 
\parbox{2.2cm}{$R^2 \times SO_2$}  &
\parbox{2.2cm}{$R^2 \times R$}  &
\parbox{2.2cm}{$R^2 \times SO_2 \times R^+$} &
\parbox{2.2cm}{$R^2 \times SO_2$} &
\parbox{2.2cm}{$R^2 \times SO_2 \times R^+$}  \\ \hline
\parbox{1.2cm}{$g=$} & 
\parbox{2.2cm}{$2+1=3$}  &
\parbox{2.2cm}{$2+1=3$}  &
\parbox{2.2cm}{$2+1+1=4$ } &
\parbox{2.2cm}{$2+1=3$} &
\parbox{2.2cm}{$2+1+1=4$}  \\ \hline
\parbox{1.2cm}{         } & 
\parbox{2.2cm}{Euclidean}  &
\parbox{2.2cm}{dilation}  &
\parbox{2.2cm}{similarity} &
\parbox{2.2cm}{Euclidean} &
\parbox{2.2cm}{similarity}  \\ \hline
\end{tabular}
\end{center}

\begin{tab}
The dimension $g$ of the global symmetry group $G$ for the structure from motion problem 
in the case $d=2$ for various cameras.
\end{tab}

\begin{center}
\begin{tabular}{|c|c|c|c|c|c|} \hline
\parbox{1.2cm}{  } & 
\parbox{2.2cm}{affine} & 
\parbox{2.2cm}{omni} &   
\parbox{2.2cm}{n.o.-omni} &  
\parbox{2.2cm}{perspective} & 
\parbox{2.2cm}{f-perspective} \\ \hline
\parbox{1.2cm}{$G=$} & 
\parbox{2.2cm}{$R^3 \times SO_3$}  &
\parbox{2.2cm}{$R^3 \times R$} &
\parbox{2.2cm}{$R^3 \times SO_3 \times R$} &
\parbox{2.2cm}{$R^3 \times SO_3$} &
\parbox{2.2cm}{$R^3 \times SO_3 \times R$} \\ \hline
\parbox{1.2cm}{$g=$} & 
\parbox{2.2cm}{$3+3=6$} &
\parbox{2.2cm}{$3+1=4$} &
\parbox{2.2cm}{$3+3+1=7$} & 
\parbox{2.2cm}{$3+3=6$} &
\parbox{2.2cm}{$3+3+1=7$} \\ \hline
\parbox{1.2cm}{        } & 
\parbox{2.2cm}{Euclidean} &
\parbox{2.2cm}{dilation} &
\parbox{2.2cm}{similarity} &
\parbox{2.2cm}{Euclidean} &
\parbox{2.2cm}{similarity} \\ \hline
\end{tabular}
\end{center}

\begin{tab}
The dimension $g$ of the global symmetry group $G$ for the structure from motion problem 
in the case $d=3$ for various cameras. The group of dilations is the group of symmetries generated
by translations and scalings. The group of similarities is generated by translations, rotations and
scalings. The Euclidean group is generated by rotations and translations.
\end{tab}

\section{Dimensional analysis}

How many points are needed to reconstruct both the points and the cameras up to a global symmetry 
transformation? This question depends on the dimension and the camera model. Assume we are in 
$d$ dimensions, have $n$ points and $m$ cameras and that the camera has $f$ internal individual 
parameters and $h$ global parameters and that a $g$-dimensional group of symmetries acts on the 
global configuration space without changing the pictures.

\begin{thm}[Dimension inequality for structure from motion]
In order that one can recover from $m$ cameras and $n$ points all the camera
parameters and all the point coordinates, it is necessary that
\begin{center}
\fbox{
$
\label{dimensionformula}
dn + fm +h \leq s \; n m + g
$
}
\end{center}
where $f$ is the dimension of the internal camera parameter space, $h$ is the
dimension of global parameters which apply to all cameras, $g$ is the 
dimension of the camera symmetry group $G$ and $s$ is the dimension of the 
retinal surface $S$. We assume that all orbits of $G$ have the same dimension.
If $dn + fm +h = s n m + g$ and the map ${\rm det}(DF)$ is not constant equal to $0$ 
then the structure from motion map $F: M^m \times N^n \to S^{mn}$
can be inverted in a locally unique way almost everywhere in $F(M^m \times N^n)$.
\end{thm}

\begin{proof}
Unknown are $dn$ point coordinates and $fm$ camera coordinates as well as $h$ global 
camera parameters. Known are $s \; nm$ data from the correspondences because the 
camera film $S$ has dimension $s$ and because there are $nm$ camera point pairs. There is a 
$g$-dimensional symmetry group $G$ to factor out. By a special case of 
Frobenius theorem (see \cite{Bazin-Boutin}),
the quotient can be parameterized by $dn + fm-g$ parameters if the orbits of $G$ have all
the same dimension. Taking pictures is a piecewise smooth map from a 
$dn + fm-g$ dimensional manifold to a $s \; n m$-dimensional manifold. 
If the inequality is not satisfied, the image of the map $F$ producing the pictures
is a sub-manifold with smaller dimension and can not cover the entire configuration space.  \\

The last statement follows from the assumption that the map $F$ is real analytic so that
the Jacobian determinant ${\rm det}(D F)$ is a real-analytic function on $M^m \times N^n$
which can be zero only on a lower dimensional subset $Z$. On $F(N^n \times M^m \setminus Z)$
the map $F$ is locally uniquely invertible by the implicit function theorem.
\end{proof}

{\bf Remarks.} \\

1) In most cases, $h=0$ and $S$ is a hypersurface so that $s=d-1$. 
An example, where $h=1$ is a perspective camera where the 
focal length $f$ is a global parameter which is the same for all cameras. If the
focal length $f$ can change from frame to frame (for example if the photographer applies
zoom manipulations while shooting the pictures), 
the parameters $f_i$ will be included in the {\bf individual parameter space}.  \\
2) We want to recover the Euclidean camera and point data. Sometimes, 
in the literature, dimension considerations are used to recover the situation 
up to the affine group or up to the projective group. 
For applications, the Euclidean structure rather than 
the affine structure (as treated in \cite{Koenderink}) is relevant. \\
3) The rather naive dimensional analysis of the inequality is 
not sufficient to solve the inverse problem.
Ambiguities can occur on lower dimensional manifolds. So, even when 
the dimension constraints are satisfied, it can happen that a locally unique reconstruction is not possible.
We investigated these ambiguities in the case of oriented omni-directional cameras
in \cite{KnillRamirezOmni}. \\
4) The structure from motion inequality can also be written by factoring out the translational symmetry first.
Assume $s=d-1$ here. 
If one of the points $O$ is fixed and kept at the origin, we have to replace $g$ with $g'=g-d$ because the point $O$
produces a coordinate origin in each plane so that $f'=f-(d-1)$. We have then 
$$ d (n-1) + f' m + h  = (d-1) (n-1) m + g'  \; . $$
But this is the same formula.  \\
5) The structure from motion inequality can be misleading. In the orthographic affine 3D case for example, 
the location of 4 points determines any other point by linearity. 
Even so adding an other point produces an additional set of 
$2m$ data points on photographs, they are redundant because they are 
already determined by the other points. This leads to examples, where the Jacobean matrix $DF$ is singular
everywhere. \\ 

In the following table, {\bf omni} abbreviates oriented omni-cameras, {\bf n.o. omni} is an 
non-oriented omni-cameras, {\bf perspective} is a perspective camera where the focal length, the distance
between the center of projection and the retinal plane is a global parameter. An {\bf f-perspective} camera 
is a camera, where between different shots, zooming is allowed and the focal length is an individual
parameter for each camera. The table summarizes the numbers $(f,g)$ for various cameras: \\

\begin{center}
\begin{tabular}{|c|c|c|c|c|c|} \hline
 (f,g,h) &   affine     &   omni     &     n.o. omni     &  perspective    & f-perspective \\ \hline
  d=2    &    (2,3,0)   &   (2,3,0)  &        (3,4,0)    &       (3,3,1)   &   (4,4,0)   \\ \hline
  d=3    &    (5,6,0)   &   (3,4,0)  &        (6,7,0)    &       (6,6,1)   &   (7,7,0)   \\ \hline
\end{tabular}
\end{center}
\begin{tab}
Overview of the dimension $f$ of $M$ and the dimension $g$ of the global symmetry group $G$
for various cameras.
\end{tab}

Let's take the case of $m=2$ and $m=3$ cameras and see what the dimension inequality predicts if
the manifold of all camera parameters matches dimension-wise the manifold of all possible camera 
point configurations. We can use the dimension inequality to count the number of points needed for
various cameras in two dimensions. First to the {\bf stereo case} with $m=2$ cameras.  \\

\begin{center}
\begin{tabular}{|c|c|c|c|c|c|} \hline
  $m=2$    &  affine    &     omni   &  n.o. omni    &   perspective   & f-perspective \\  \hline
   d=2     &    -       &       -    &     -         &        -        &    -          \\  \hline
   d=3     &    4       &       2    &     5         &        7        &      8        \\  \hline
\end{tabular}
\end{center}
\begin{tab}
Bounds given by the dimension inequality for the number $n$ of points needed with $m=2$ cameras.
No planar camera pair can recover structure from motion. The 7 correspondences apply if both 
cameras have the same focal length. If the focal length can change, we need 8 points.
\end{tab}

For $m=3$ cameras, the situation improves, especially in the plane:

\begin{center}
\begin{tabular}{|c|c|c|c|c|c|} \hline
  $m=3$         &  affine       &    omni      & omni unoriented   &  perspective         & perspective w. zoom \\  \hline
     d=2        &    3          &     3        &       5           &    6                 &    8        \\  \hline
     d=3        &    3          &     2        &       4           &    4                 &    5        \\  \hline
\end{tabular}
\end{center}
\begin{tab}
The number $n$ of points needed with $m=3$ cameras. If the number is larger for $d=2$ 
than for $d=3$ which is always the case except for orthographic affine cameras, 
this means that the reconstruction in space will need a noncoplanarity assumption.
\end{tab}

In the affine orthographic 3D case, the classical Ullman theorem states that $4$ points suffice to 
reconstruct points uniquely up to a reflection.  
The dimension formula shows that $n=3$ is enough for a {\bf locally} unique
reconstruction. Explicit reconstruction formulas can be given 
\cite{KnillRamirezUllman}. An additional 4'th point reduces the number
of discrete ambiguities but over-determines the system of equations so that reconstruction is only 
possible on a lower dimensional manifold. \\

The dimension formula only tells hat happens generically. 
For example, if the camera-point configurations are contained in one single plane, 
the larger 2D numbers apply. 
Even so the dimensional analysis shows that two points should be enough in space, we need 
three points if the situation is coplanar and noncolinearity conditions are needed to eliminate all 
ambiguities. We will see with counter examples that these results are sharp. The dimension formula gives
a region in the $(n,m)$ plane, where the structure from motion problem can not have a unique
solution. We call these regions {\bf forbidden region} of the structure from motion problem. 

\section{Orthographic cameras}

In dimension $d$, an affine orthographic 
camera is determined by $f={\rm dim}(SO_d) + {\bf R}^{d-1} = d (d-1)/2 + (d-1)$ 
parameters and a point by $d$ parameters. We gain $(d-1)$ coordinates for each point-camera pair. 
The global Euclidean symmetry of the problem (rotating and translating the
point-camera configuration does not change the pictures) gives us the {\bf structure from the 
motion inequality for orthographic cameras}
$$  n d + m [ d (d-1)/2 + (d-1)] \leq (d-1) n m + d + d (d-1)/2 $$
which for $m=3$ and $d=2$ this gives $2 n + 6 \leq 3 n  + 2 + 1$ which means that $n=3$ is sharp. 
For $m=3$ and $d=3$ we are left with $3 n + 15 \leq 6 n + 3 + 3$ which means $n=3$ is sharp.
For $m=2$ and $d=3$ we get $3 n + 10 = 16+ 6$ which indicates $n=4$ points is sharp. Note that
unlike for $(m,n)=(3,3)$, where a locally unique reconstruction is possible, this is not the 
case for $(m,n) =(2,4)$. For $m=2$ orthographic cameras in space and arbitrarily many points, there
are always deformation ambiguities. \\

\begin{center}
\parbox{15.5cm}{
\parbox{7.5cm}{\scalebox{0.45}{\includegraphics{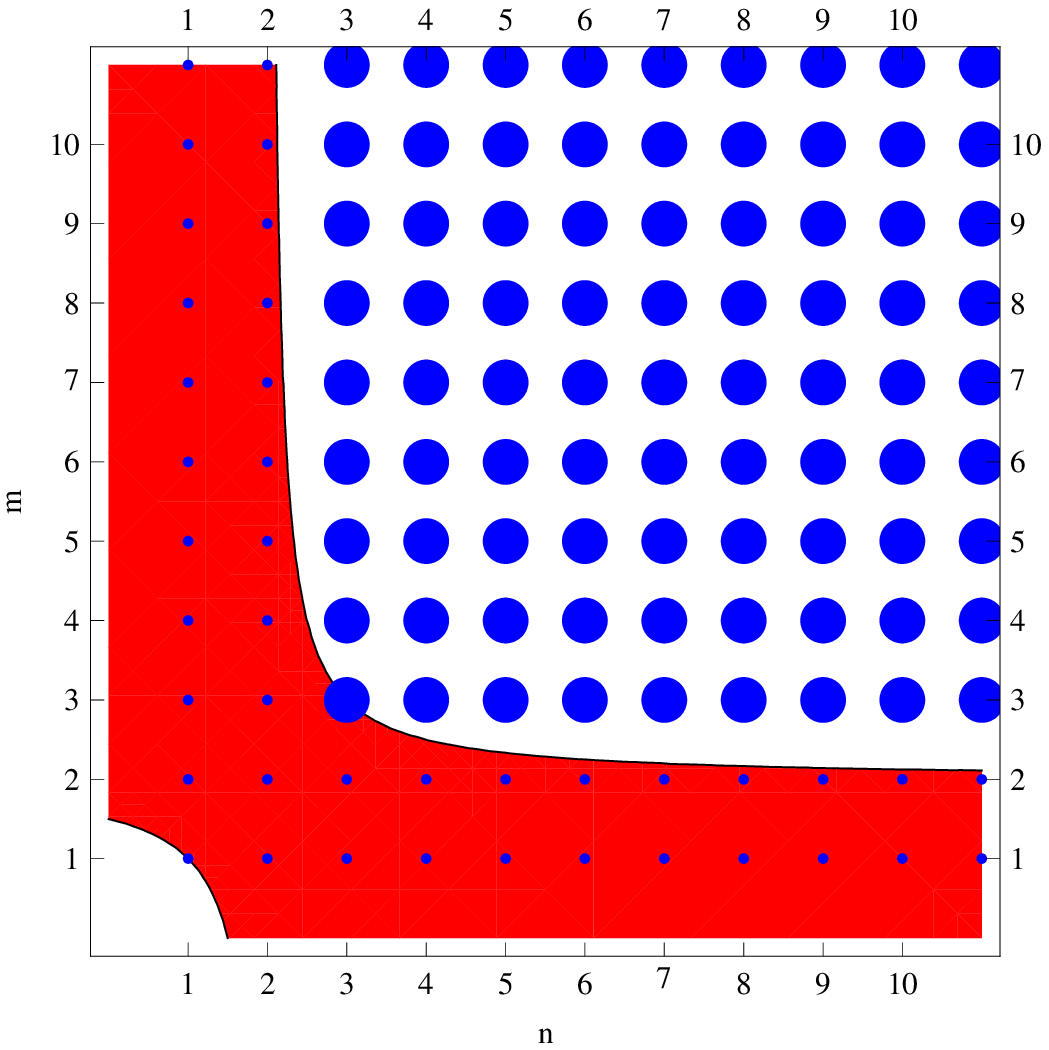}}}
\parbox{7.5cm}{\scalebox{0.45}{\includegraphics{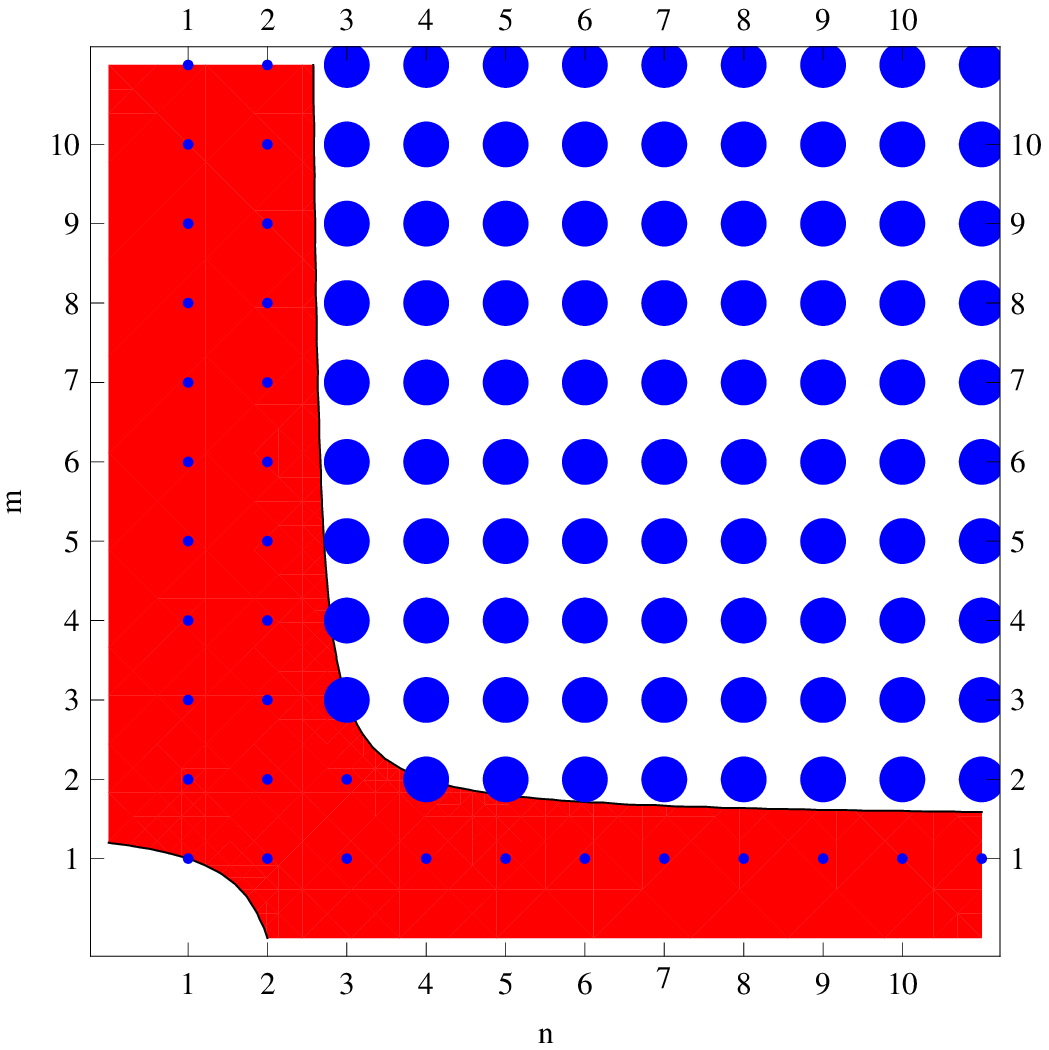}}}
}
\end{center}
\begin{center}
\parbox{15.5cm}{
\parbox{7.5cm}{Orthographic \\ (d,f,g) = (2,2,3)}
\parbox{7.5cm}{Orthographic \\ (d,f,g) = (3,5,6)}
}
\end{center}
\begin{fig}
The forbidden region in the $(n,m)$ plane for affine orthographic cameras.
In space, this is the situation of the celebrated Ullman theorem. 
\end{fig}

\section{Perspective cameras}

\begin{center}
\parbox{15.5cm}{
\parbox{7.5cm}{\scalebox{0.45}{\includegraphics{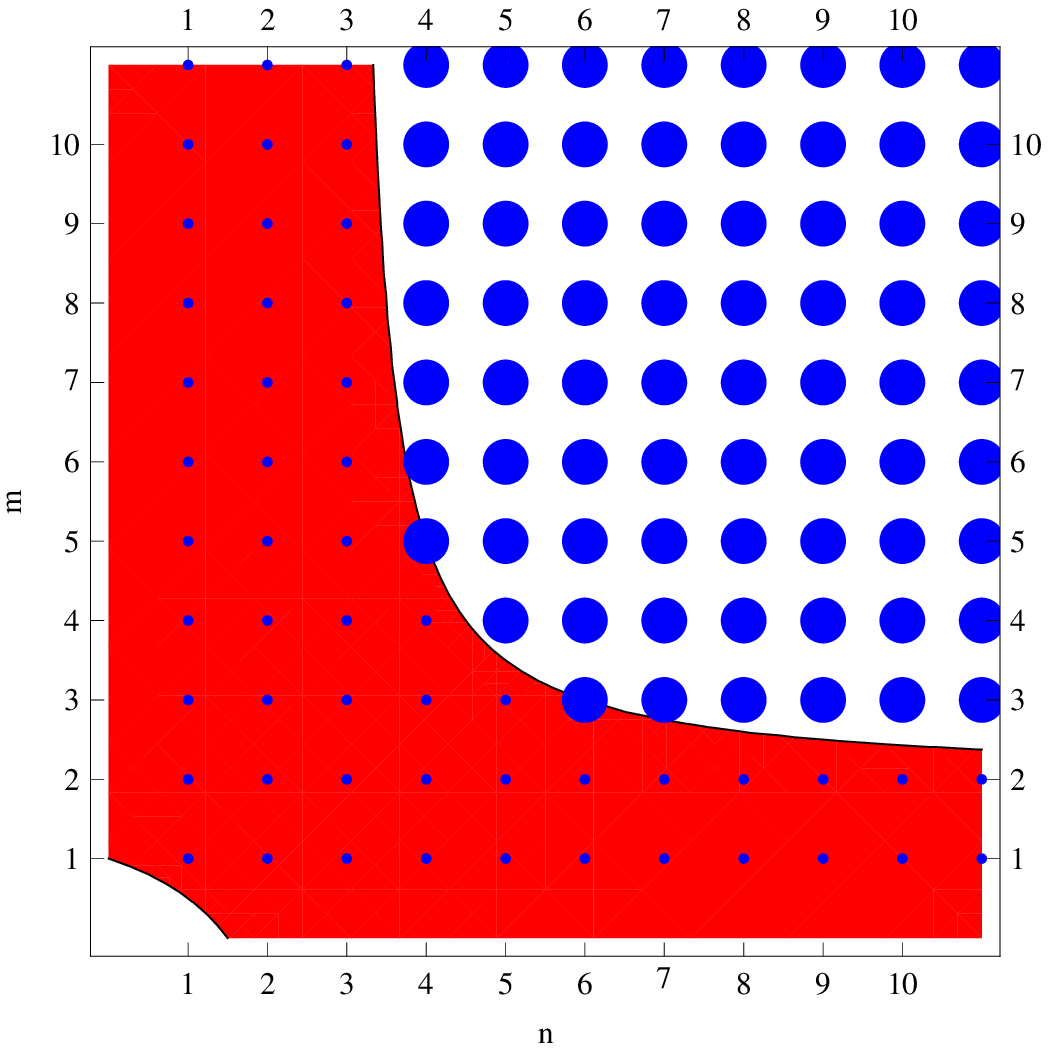}}}
\parbox{7.5cm}{\scalebox{0.45}{\includegraphics{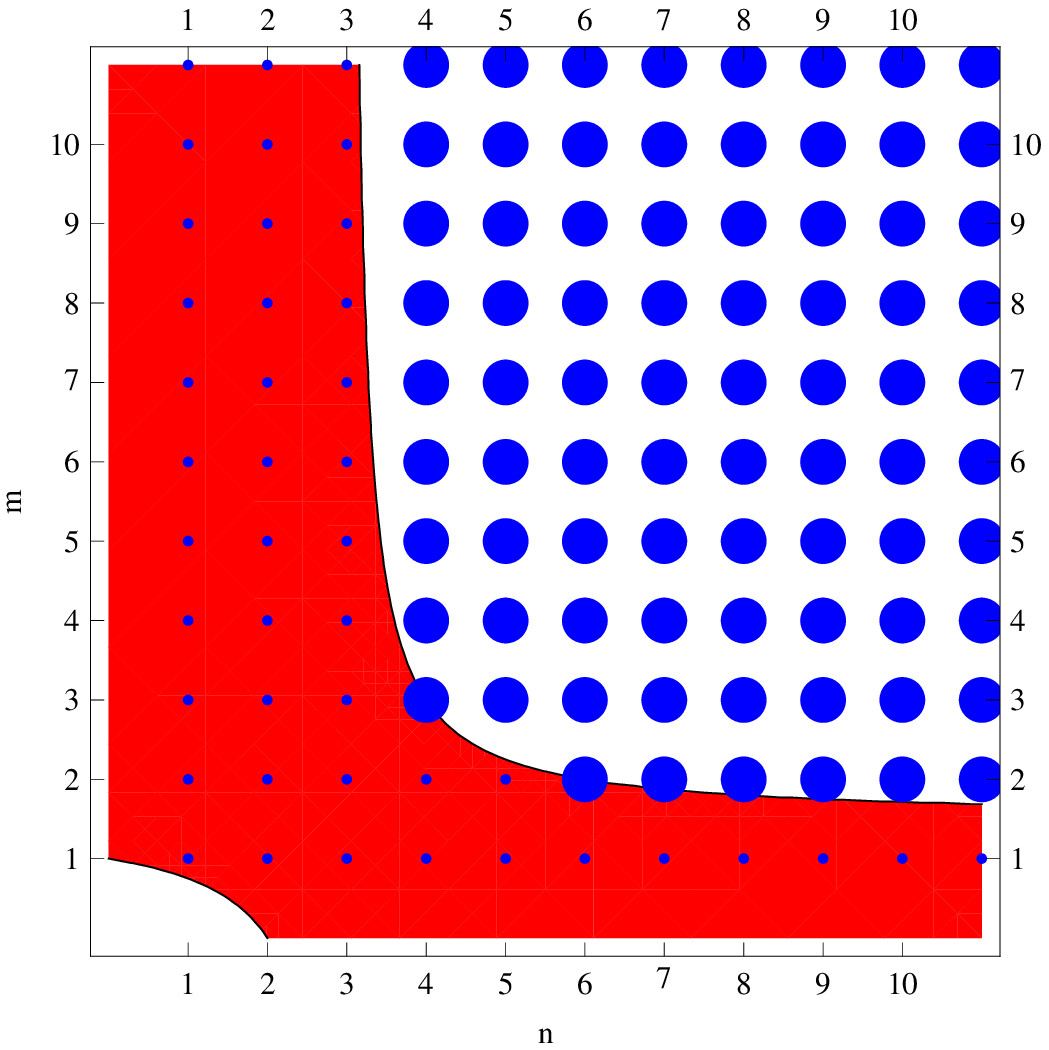}}}
}
\end{center}
\begin{center}
\parbox{15.5cm}{
\parbox{7.5cm}{Perspective \\ (d,f,g) = (2,3,3)}
\parbox{7.5cm}{Perspective \\ (d,f,g) = (3,6,6)}
}
\end{center}
\begin{fig}
The forbidden region in the $(n,m)$ plane for perspective cameras with given known
focal length. We see that the stereo case $m=2$ needs $n=6$ points. Since it is a 
border line case, we have a locally unique reconstruction in general. 
\end{fig}

We see that $n=6$ points are needed in the stereo case with $m=2$ cameras. This assumes 
that the focal length is known. If the focal length is not known, then $n=7$ points 
are needed. This is the situation first considered by Chasles \cite{Chasles}. 
One can deduce the number
$7$ also differently: fix the first camera plane as the $xy$ plane and take the focal 
point on the $z$ axes. This needs 1 parameter. The second camera needs 6 parameters,
3 for the focal point and 3 for the plane orientation. For each point, we add $3$ unknowns but
gain $m \cdot (d-1) = 2  \cdot 2=4$ coordinates. That means for every added point, we gain 
one parameter. So, 7 points are enough. 

\begin{center}
\parbox{15.5cm}{
\parbox{7.5cm}{\scalebox{0.45}{\includegraphics{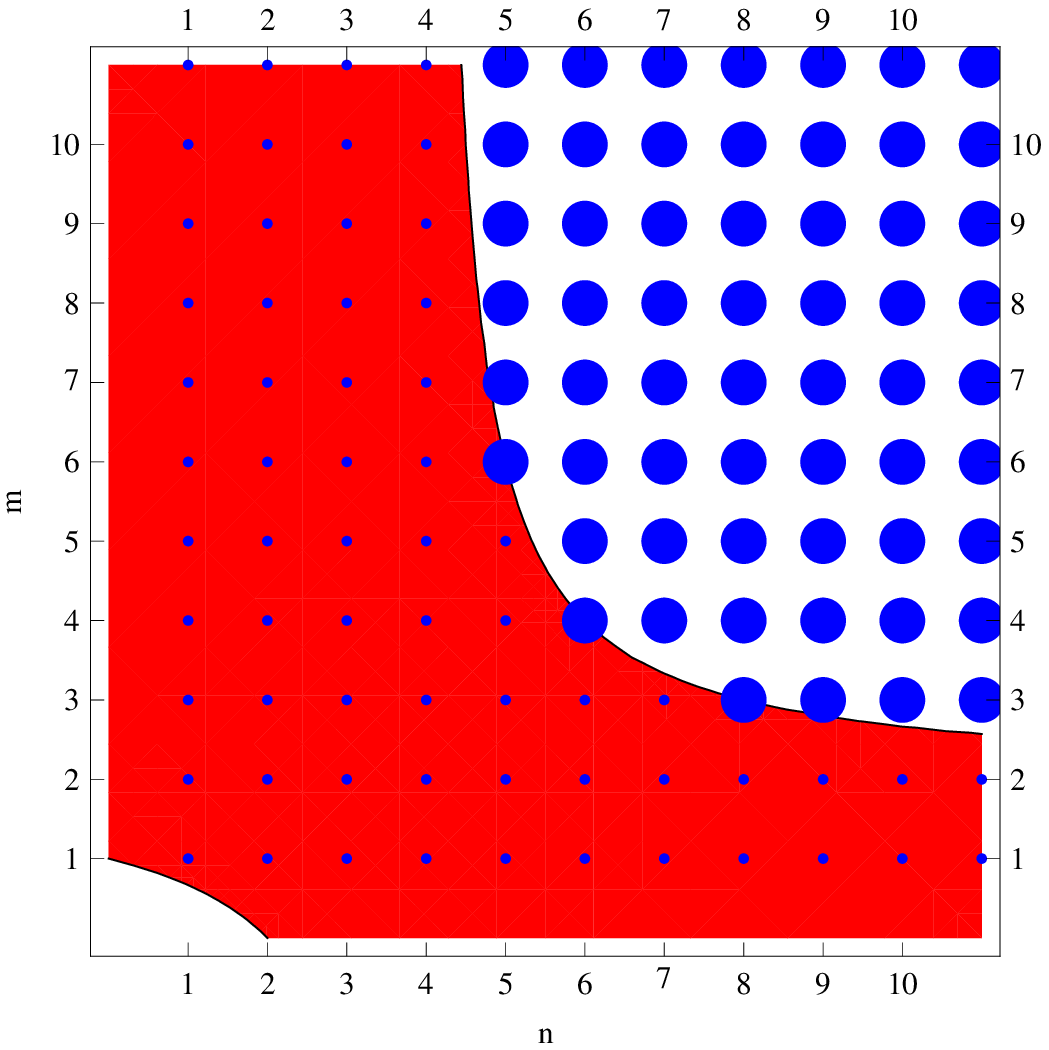}}}
\parbox{7.5cm}{\scalebox{0.45}{\includegraphics{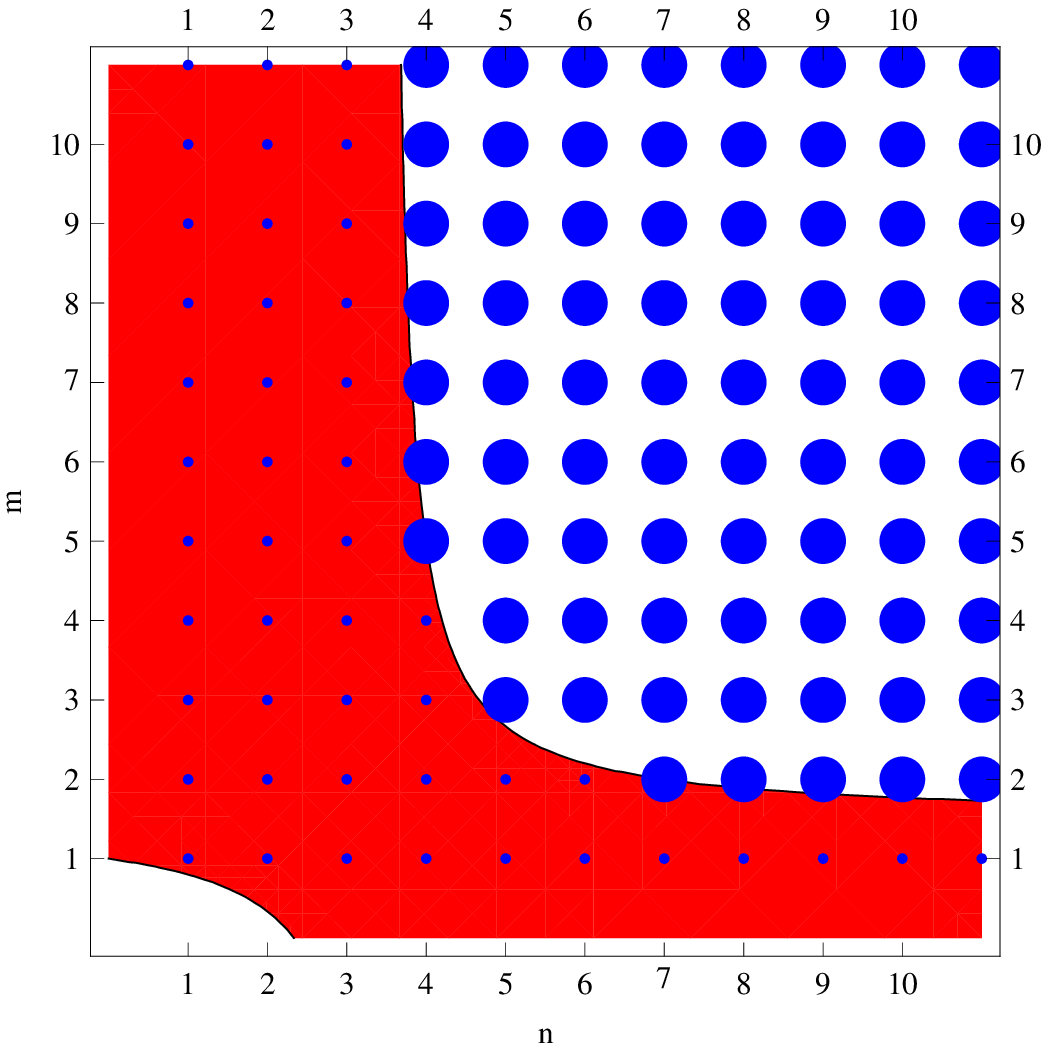}}}
}
\end{center}
\begin{center}
\parbox{15.5cm}{
\parbox{7.5cm}{Perspective with zoom \\ (d,f,g) = (2,4,4)}
\parbox{7.5cm}{Perspective with zoom \\ (d,f,g) = (3,7,7)}
}
\end{center}
\begin{fig}
The forbidden region in the $(n,m)$ plane for perspective cameras with unknown
focal length. The photographer can change the focal length from one camera to 
the next. 
\end{fig}

\section{Four cameras}

Finally, lets look how many points we expect for $m=4$ cameras: \\

\begin{center}
\begin{tabular}{|c|c|c|c|c|c|} \hline
  $m=4$         &  affine       &    omni      & omni unoriented   &  perspective         & perspective w. zoom \\  \hline
     d=2        &    3          &     3        &       4           &    5                 &    7        \\  \hline
     d=3        &    3          &     2        &       4           &    4                 &    5        \\  \hline
\end{tabular}
\end{center}
\begin{tab}
The number $n$ of points needed with $m=4$ cameras.  So, seeing 4 points with 4 perspective cameras should
allow us to reconstruct both the cameras and points in general. Again, we need in general more cameras in 
coplanar situations. 
\end{tab}

\section{Oriented omni-directional cameras}

How many points and cameras do we need for oriented omni-directional cameras?

\begin{center}
\parbox{15.5cm}{
\parbox{7.5cm}{\scalebox{0.45}{\includegraphics{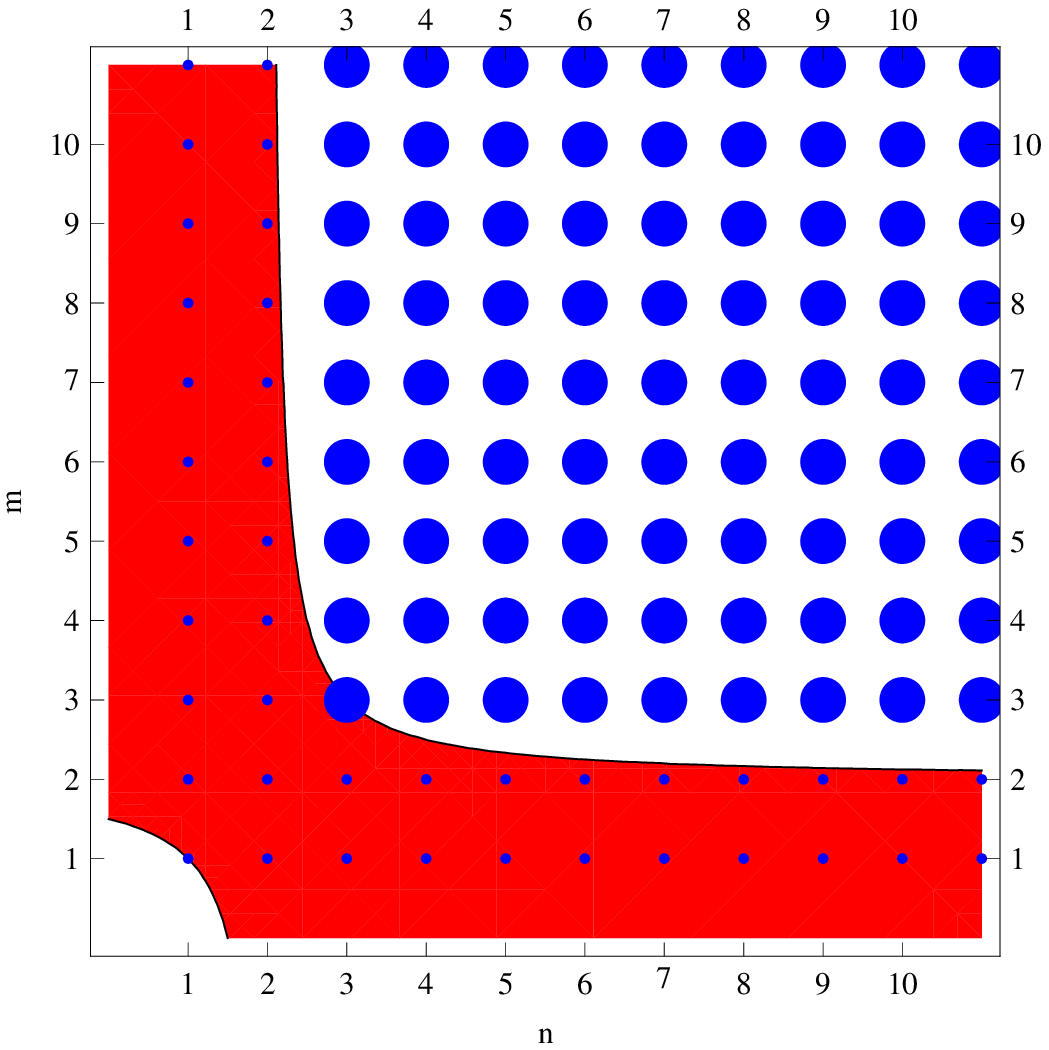}}}
\parbox{7.5cm}{\scalebox{0.45}{\includegraphics{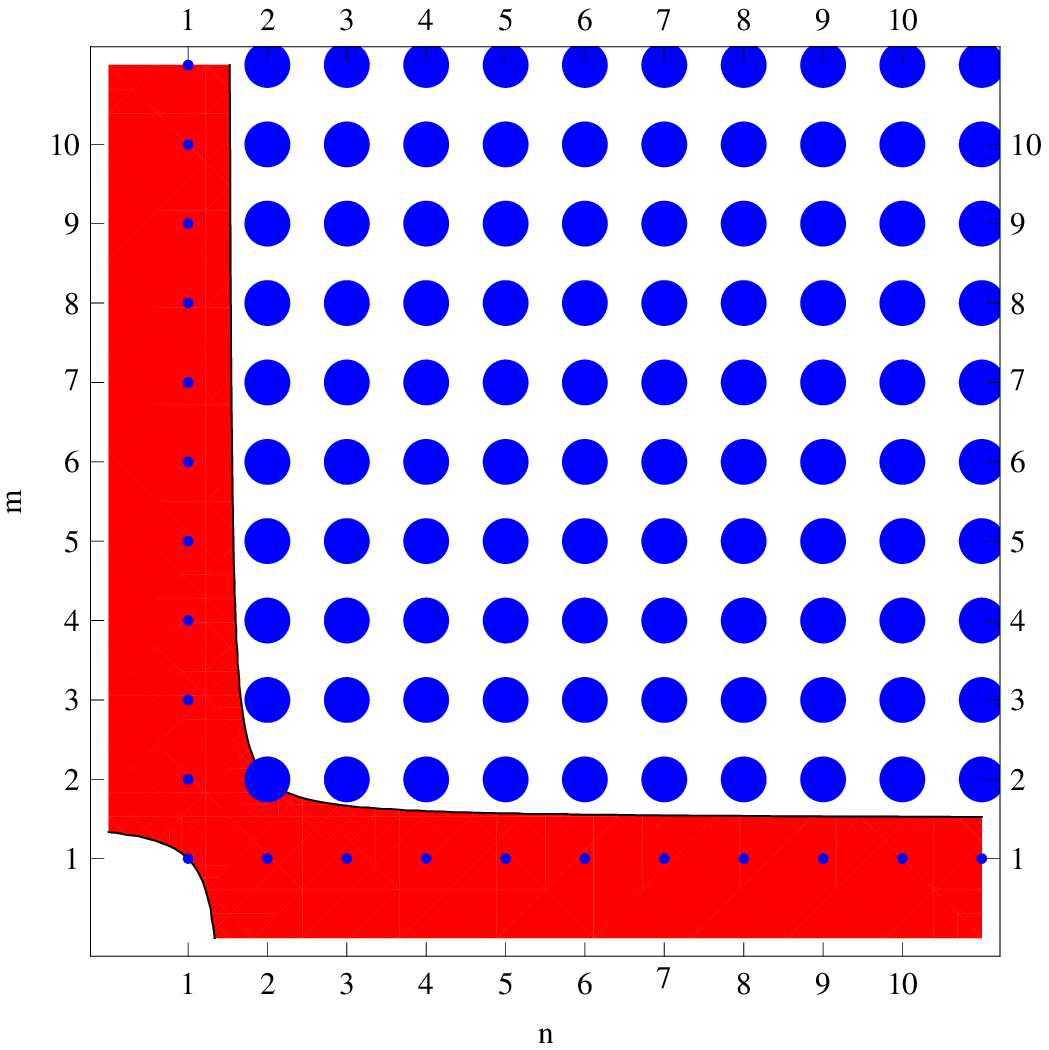}}}
}
\end{center}
\begin{center}
\parbox{15.5cm}{
\parbox{7.5cm}{Oriented Omni  \\ (d,f,g) = (2,2,3)}
\parbox{7.5cm}{Oriented Omni  \\ (d,f,g) = (3,3,4)}
}
\end{center}
\begin{fig}
The forbidden region in the $(n,m)$ plane for oriented omni-directional cameras.
\end{fig}

Let's compare the two sides of the dimension formula in the oriented planar omni-directional
case $(d,f,g) = (2,2,3)$:  \\

\begin{tabular}{|c|c|c|c|c|} \hline
   Cameras &  Points        &  equations nm    &  unknowns 2(n+m)-3  &   unique ?  \\ \hline 
   m=1     &     n          &     n            &    2n-1             &   no, one camera ambiguities   \\
   m=2     &     n          &    2n            &    2n+1             &   no, two camera ambiguities  \\
   m=3     &    $n=2$       &     6            &     7               &   no, two point ambiguities  \\
   m=3     &    $n \geq 3$  &    3n            &    2n+3             &   yes, if no ambiguities\\
   m=4     &    $n \geq 3$  &    4n            &    2n+5             &   yes, if no ambiguities\\ \hline
\end{tabular}

\parbox{14.8cm}{
\parbox{7cm}{\scalebox{0.60}{\includegraphics{goodpoints/orientedomni2d.ps}}}
\parbox{7cm}{
\begin{fig}
In the plane $d=2$ with camera parameters $(f,g)=(2,3)$. The reconstruction region 
$dn+fm \leq (d-1) n m + g$ is given by $mn-2m-2n+3 <0$. The situation $(n,m)=(3,3)$ is 
the only borderline case. Also in all other cases $n,m \geq 3$ we have more or equal 
equations than unknowns.
\end{fig}
}
}

\section{Non-oriented omni-cameras}

For non-oriented spherical cameras in the plane, the camera manifold $M$ is the three dimensional 
space ${\bf R}^2 \times SO_1$ and the global symmetry group is the $g=4$-dimensional group of
similarities. The structure from motion inequality reads
$$ 2 n + 3 m = n m + 4    \; . $$
For $m=3$ cameras, we need $n=5$ points, for $m=4$ cameras, we need $4$ points. Because
both of these cases are borderline cases, we expect a locally unique reconstruction almost 
everywhere. \\

In space, where $d=3, f=6$ and $g=7$, the inequality is
$$ 3 n + 6 m = 2 n m + 7  \; . $$ 
For $m=3$ cameras we need $4$ points, for $m=2$ cameras, we need $n=5$ points. 

\begin{center}
\parbox{15.5cm}{
\parbox{7.5cm}{\scalebox{0.45}{\includegraphics{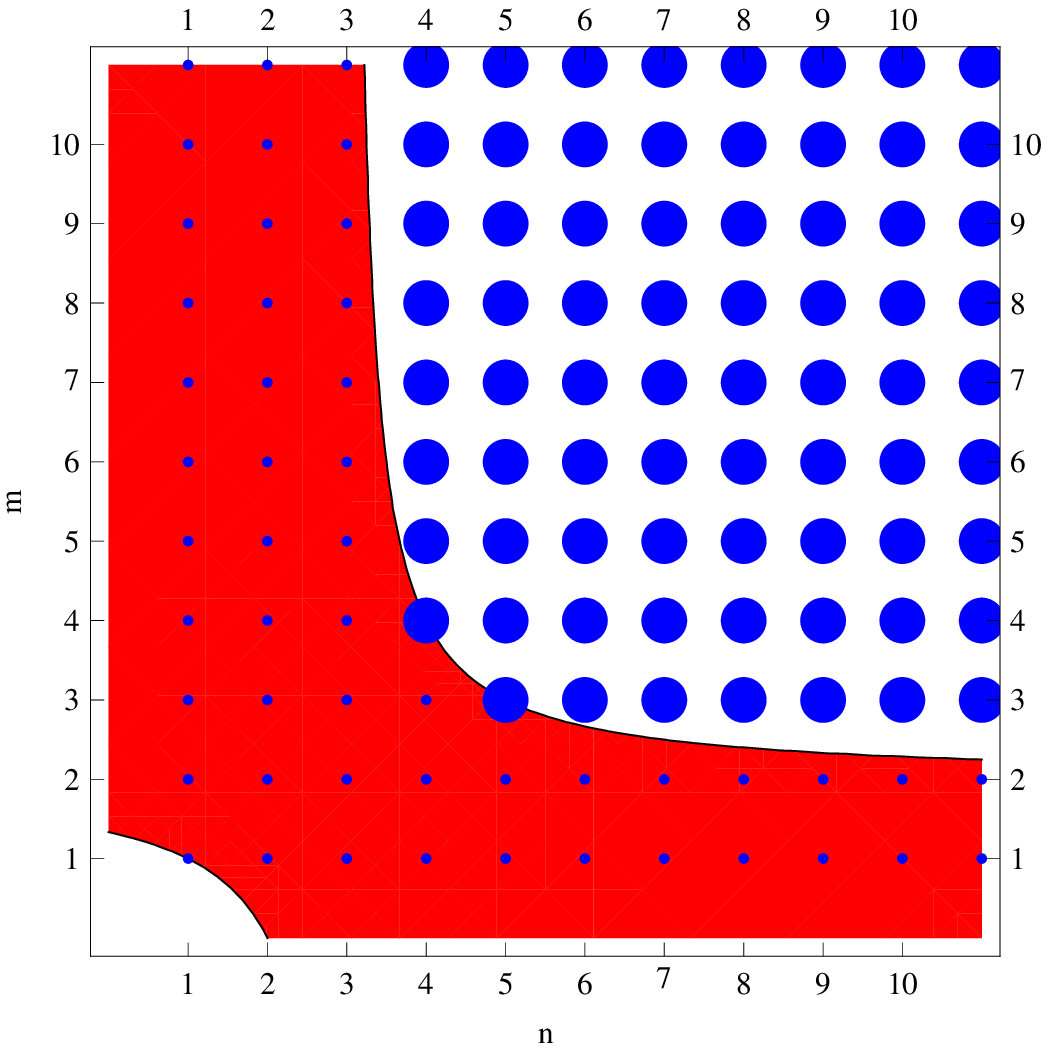}}}
\parbox{7.5cm}{\scalebox{0.45}{\includegraphics{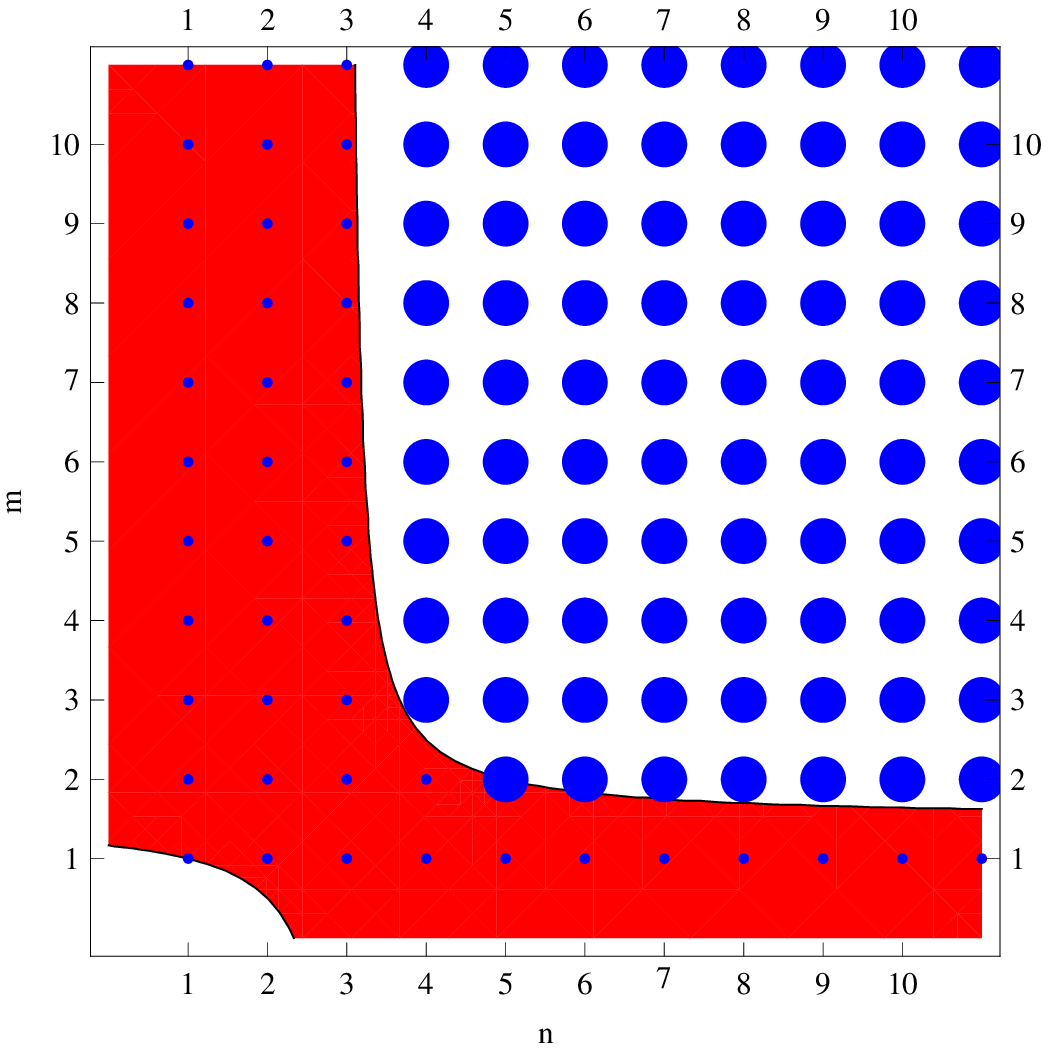}}}
}
\end{center}
\begin{center}
\parbox{15.5cm}{
\parbox{7.5cm}{Non oriented Omni  \\ (d,f,g) = (2,3,4)}
\parbox{7.5cm}{Non oriented Omni  \\ (d,f,g) = (3,6,7)}
}
\end{center}
\begin{fig}
The forbidden region in the $(n,m)$ plane for oriented omni-directional cameras.
\end{fig}

{\bf Remarks}.  \\
1) It seems unexplored, under which conditions the construction is
unique for unoriented omni-cameras. Due to the nonlinearity
of the problem, this is not as simple as in the
oriented case \cite{KnillRamirezOmni}.
The equations for the unknown point positions $P_i = (x_i,y_i)$ and camera
positions $Q_j = (a_j,b_j)$ and camera angles $\alpha_j$ are 
$$  \sin(\theta_{ij}+\alpha_j) (x_i - a_j) = \cos(\theta_{ij}+\alpha_j) ( y_i - b_j)  \; . $$

2) For omni-directional cameras in space which all point in the same direction but turn around this axis,
the dimension analysis is the same. We can first compute the first two coordinates and then the third coordinate.
When going to the affine limit, these numbers apply to camera pictures for which we know one direction.
This is realistic because on earth, we always have a gravitational direction. 

\section{A codimension 2 camera}

We quickly look at an example of a camera, where the retinal surface is not a hypersurface. The camera $Q$
is given by a line $S$ in space. 
The map $Q$ is the orthographic projection of a point $P$ onto $S=S(Q)$. 

How many points do we need for a reconstruction with $3$ cameras?  We have $d=3$ and $s=1$. Because
a line in space is determined by a point and a direction, the dimension $f$ of the camera manifold is 
$f=3$. The global symmetry group consists of Euclidean transformations, which gives $g=6$. The 
structure from motion inequality tells
$$    3 n + 3 m = n m + 6  \; . $$
For $m=4$ cameras, we need $n=6$ points. 

\begin{center}
\parbox{7.5cm}{\scalebox{0.45}{\includegraphics{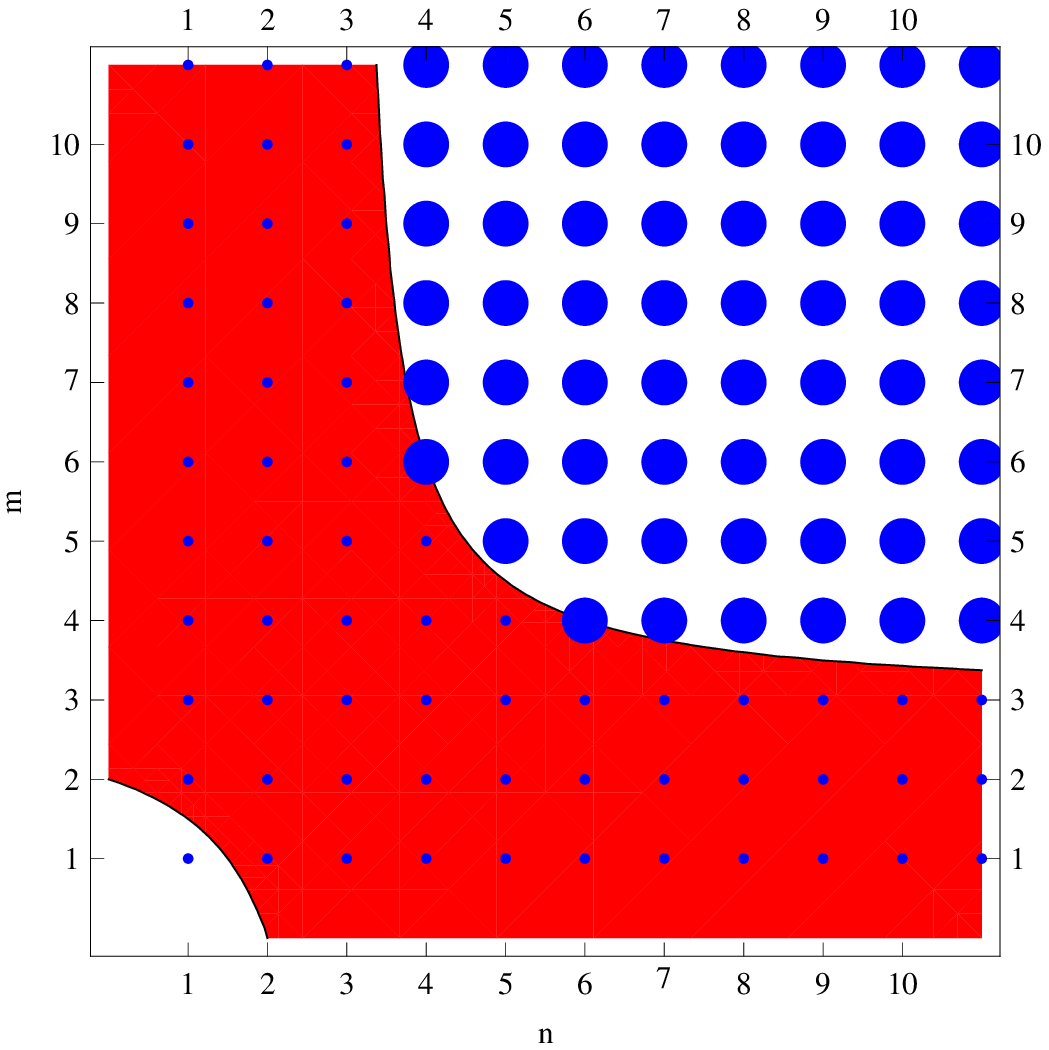}}}
\end{center}
\begin{center}
\parbox{15.5cm}{
\parbox{7.5cm}{The forbidden region for codimension 2 cameras  \\ (d,f,g,s) = (3,3,6,1).}
}
\end{center}
\begin{fig}
The forbidden region in the $(n,m)$ plane for line cameras in space. 
\end{fig}

\section{Moving bodies}

We consider now the situation where a camera moves through a scene, in which the bodies themselves 
can change location with time. This setup justifies the relatively abstract definition of a camera as
a transformation $Q: N \to N$ satisfying $Q^2=Q$ such that the image $F(N)$ is isomorphic to 
a lower dimensional manifold $S$. \\

Examples of structure from motion problems with moving bodies are a camera mounted on a car
moving in a traffic lane, where the other cars as well as part of the street define points. 
An other setup could be a movable camera in a football stadium where the points are the 
players which move on a football field. A historically important example is the 
earth observing other planets and stars. The last example is historically the oldest structure
from motion problem: it is the task to reconstruct the position of our earth within the other
structures of the universe. \\

If points can move, we still have $n m$ equations and a global $g$ dimensional 
symmetry group but now $3nk+3mf$ 
unknown parameters. The dimension formula still applies. But now, the dimension of the space $N$ is
$d (k+1)$. The point space $M$ is larger and the retinal plane $S$ has a much lower dimension than $M$. 
Let's formulate it as a lemma:

\begin{lemma}[Dimension formula for moving bodies]
If the motion of every point in the $d$ dimensional scene is described with a Taylor expansion 
of the order $k$, then the following condition has to be satisfied
$$ d n (k+1) + m f +h \leq n m (d-1) + g  $$
so that we can hope to reconstruct the motion of the points simultaneously to the
motion of the camera.
\end{lemma}

For example: assume we know that all points $P_i$ move on circles in the plane with 
constant angular velocity. The point configuration space $N$ is ${\bf R}^4$, because for every point,
we have to specify the center of the circle, as well as the vector from the center to the point.
How many points do we need for $m$ cameras? \\
For affine or oriented omni cameras with $(f,g)=(2,3)$, the structure from motion inequality gives
$$ n 4 + m 3 = n m + 4  \; . $$
We need at least $m=5$ cameras to allow a reconstruction. The inequality assures
us that with 4 pictures, a unique reconstruction is impossible. For $m=5$ cameras, we need at least $n=11$ 
points.  For $m=6$ cameras, we need at least $n=7$ points.  
If we observe a swarm of 11 points with 5 camera frames, we expect a reconstruction of the moving
points and the cameras. 

\parbox{14.8cm}{
\parbox{7cm}{\scalebox{0.70}{\includegraphics{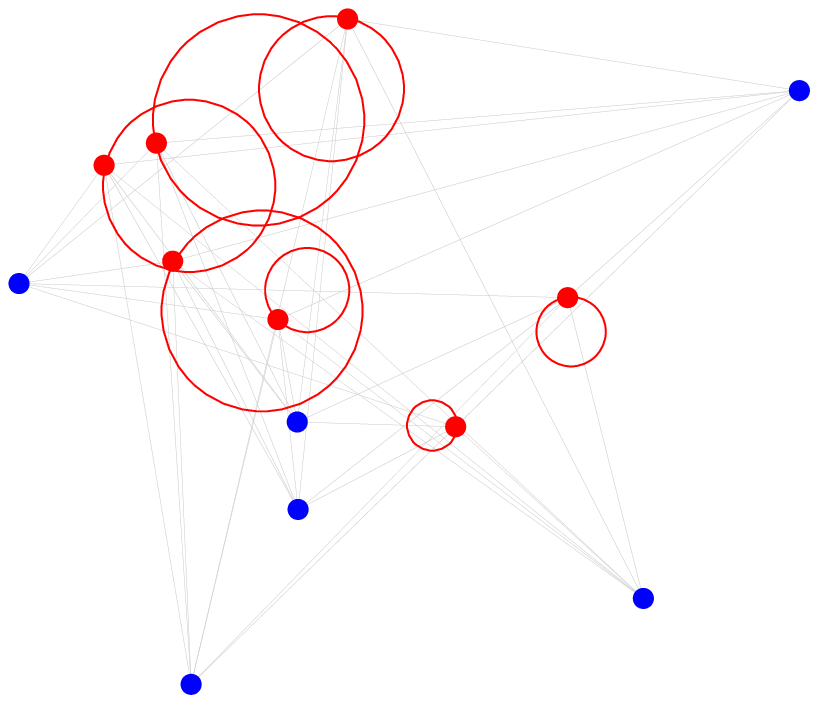}}}
\parbox{7cm}{
\begin{fig}
$m=6$ oriented omni-directional cameras observe $n=7$ points moving on circles. 
The reconstruction recovers the circles and the cameras up to a global rotation, 
translation and scaling. 
The camera takes $m$ pictures at times $t_1,...,t_m$. We observe the points $P_i(t_j)$ if the times
$t_i$ are known. 
\end{fig}
}
}

\vspace{12pt}
\bibliographystyle{plain}
\bibliography{3dimage}
\end{document}